\documentclass[onefignum,onetabnum]{siamonline190516}
\usepackage{braket,amsfonts}
\usepackage{adjustbox}

\usepackage{tikz}
\usetikzlibrary{positioning,chains,fit,shapes,calc}

\definecolor{tempcolor}{rgb}{1,0.2,0.3}

\usetikzlibrary{calc, chains,
	fit,
	positioning,
	shapes}

\definecolor{myblue}{RGB}{80,80,160}
\definecolor{mygreen}{RGB}{80,160,80}

\usepackage{array}

\usepackage[caption=false]{subfig}
\usepackage{pgfplots}

\newsiamthm{claim}{Claim}
\newsiamremark{remark}{Remark}
\newsiamremark{hypothesis}{Hypothesis}
\crefname{hypothesis}{Hypothesis}{Hypotheses}

\usepackage{algorithmic}

\usepackage{graphicx,epstopdf}

\Crefname{ALC@unique}{Line}{Lines}

\usepackage{amsopn}

\usepackage{xspace}
\usepackage{bold-extra}
\usepackage[most]{tcolorbox}

\colorlet{texcscolor}{blue!50!black}
\colorlet{texemcolor}{red!70!black}
\colorlet{texpreamble}{red!70!black}
\colorlet{codebackground}{black!25!white!25}

\lstdefinestyle{siamlatex}{%
  style=tcblatex,
  texcsstyle=*\color{texcscolor},
  texcsstyle=[2]\color{texemcolor},
  keywordstyle=[2]\color{texemcolor},
  moretexcs={cref,Cref,maketitle,mathcal,text,headers,email,url},
}

\tcbset{%
  colframe=black!75!white!75,
  coltitle=white,
  colback=codebackground, 
  colbacklower=white, 
  fonttitle=\bfseries,
  arc=0pt,outer arc=0pt,
  top=1pt,bottom=1pt,left=1mm,right=1mm,middle=1mm,boxsep=1mm,
  leftrule=0.3mm,rightrule=0.3mm,toprule=0.3mm,bottomrule=0.3mm,
  listing options={style=siamlatex}
}

\newtcblisting[use counter=example]{example}[2][]{%
  title={Example~\thetcbcounter: #2},#1}

\newtcbinputlisting[use counter=example]{\examplefile}[3][]{%
  title={Example~\thetcbcounter: #2},listing file={#3},#1}

\newtheorem{defi}{Definition}[section]
\newtheorem{ex}[defi]{Example}

\DeclareTotalTCBox{\code}{ v O{} }
{ 
  fontupper=\ttfamily\color{black},
  nobeforeafter,
  tcbox raise base,
  colback=codebackground,colframe=white,
  top=0pt,bottom=0pt,left=0mm,right=0mm,
  leftrule=0pt,rightrule=0pt,toprule=0mm,bottomrule=0mm,
  boxsep=0.5mm,
  #2}{#1}

\def\F{{\mathbf F}}
\def\b{{\mathbf b}}
\def\x{{\mathbf x}}
\def\ve{{\mathbf v}}
\def\v{{\mathbf v}}
\def\u{{\mathbf u}}
\def\h{{\mathbf h}}
\def\k{{\mathbf k}}
\def\c{{\mathbf c}}
\def\d{{\mathbf d}}
\def\p{{\mathbf p}}
\def\q{{\mathbf q}}
\def\e{{\mathbf e}}

\patchcmd\newpage{\vfil}{}{}{}
\begin{tcbverbatimwrite}{tmp_\jobname_header.tex}
\title{Clustering, multicollinearity, and singular vectors}

\author{Hamid Usefi\thanks{Department of Mathematics and Statistics, Memorial University of Newfoundland, 
		St. John's, NL, Canada, A1C 5S7, 
		 (\email{usefi@mun.ca}).}}

\headers{Clustering, multicollinearity, and singular vectors}{Hamid Usefi}
\end{tcbverbatimwrite}
\input{tmp_\jobname_header.tex}

\ifpdf
\hypersetup{ pdftitle={Clustering, multicollinearity, and singular vectors} }
\fi

\begin{document}
\maketitle

\begin{tcbverbatimwrite}{tmp_\jobname_abstract.tex}
\begin{abstract}
	 Let $A$ be a matrix with its pseudo-matrix $A^{\dagger}$ and set   $S=I-A^{\dagger}A$. We prove that, after re-ordering the columns of $A$, the matrix $S$ has a block-diagonal form where each block corresponds to a set of 
	linearly dependent columns. This allows us to identify redundant columns in $A$. We explore  some applications in supervised and unsupervised learning, specially feature selection, clustering, and sensitivity of solutions of least squares solutions.

\end{abstract}


\end{tcbverbatimwrite}
\input{tmp_\jobname_abstract.tex}

\section{Introduction}
\label{sec:intro}

 In this paper, we tackle the problem of  identifying  linearly dependent columns of a matrix $A$. In other words, we identify clusters of columns of $A$ such  that  every two columns in a cluster are part of a dependence relation. Variations of this problem have been extensively studied.   Given a matrix $A\in {\bf R}^{m\times n}$, a vector $\b\in {\bf R}^m$, and $\epsilon > 0$, Natarajan \cite{natarajan1995sparse} considers sparse approximate solutions to $A\x=\b$, that is  compute a vector $\x$ that satisfies  $\|A\x - \b\|_{2} \leq \epsilon $ if such exists, such that $\x$ has the fewest number of non-zero entries over all such vectors.
   This problem can also be viewed as multicollinearity and   ``subset selection'' in statistical
 modeling and has been extensively explored \cite{golub2012matrix, civril2014column,  Golub76rankdegeneracy}.  Multicollinearity arises in many contexts, including  regression \cite{alin2010multicollinearity, vatcheva2016multicollinearity}, ecology \cite{dormann2013collinearity, farrell2019machine}, and machine learning \cite{tamura2017best, xu2019machine}.

Let $A$ be an  $m\times n$ matrix of rank $\rho\leq \min(m,n)$ and denote by $\mathcal{T}$ the set of all columns of $A$. Let 
$\tau\subseteq \mathcal{T}$ and denote by $\mathcal{V}$ and $\mathcal{V}'$ the subspaces spanned by $\tau$ and $\mathcal{T}\setminus\tau$, respectively. We say $\tau$ is \textit{maximally dependent} if  $\tau$ is dependent and $\mathcal{V}\cap \mathcal{V}'=0$. 

Our main objective is to identify maximally dependent subsets of columns of $A$; these subsets can be viewed as clusters.  Let  $A^{\dagger}$  be the pseudo-matrix of $A$ and  set $S=I-A^\dagger A$. We prove in Theorem \ref{main1} that if 
$\tau_1$ and $\tau_2$ are maximally dependent subsets of columns of $A$ then $S_{i,j}=0$, for every 
$\F_i\in \tau_1$ and $\F_j\in \tau_2$. This implies that $S$ is similar to a block-diagonal matrix, this can be seen by    moving  and  grouping the columns within the same  cluster together.
We make a  critical observation in Lemma \ref{SP}    that $S$ is the same as  the orthogonal projection $P$  onto the null space   of $A$.

Our ultimate goal is to prove that the blocks on the diagonal of $S$ (after re-labeling columns of $A$) correspond to 
maximally dependent subsets of columns of $A$. What we need to show is that these blocks themselves do not decompose into block-diagonal matrices. To do this, we define a graph $G$ where the nodes are columns of  $A$ and there is an edge between $\F_i$ and $\F_j$ if and only if $S_{i,j}\neq 0$. In Theorem \ref{graph}, we prove that 
if $\tau\subseteq \mathcal{T}$ is a maximally dependent subset, then the  sub-graph $G_{\tau}$ of $G$ corresponding to $\tau$   is connected. We deduce that the connected components of $G$ corresponds to clusters of linearly dependent columns of $A$.

Supervised learning is a central problem in machine learning and data mining. In this process, a mathematical/statistical model is trained and generated based on a 
pre-defined number of instances (train data) and is tested against the remaining (test data).
 Let  $D=[A\mid \mathbf{b}]$ be a dataset where $\mathbf{b}$ is the class label and $A$ is a matrix each of its rows is a sample (instance). 
 The columns of $A$ are referred to features or attributes. Feature selection is the process of selecting a small subset of features that can be used to build a model to predict  $\mathbf{b}$  \cite{RJ+97}.
  In Section \ref{applications}, we explain how the signature matrix of $D$ can be used to develop a feature selection algorithm. We shall also mention applications to clustering and un-supervised learning.

We also explore an application to the sensitivity of solutions of least square problems to perturbations. 
Chandrasekaran and Ipsen in \cite{chandrasekaran1995sensitivity} investigated the the errors in individual components of the solution to systems of linear equations and linear least squares problems of full column rank. They proposed  ``componentwise condition numbers'' to measure the sensitivity of each solution component to perturbations and showed   that any linear system has at least one solution component whose sensitivity to perturbations is proportional to the condition number of the matrix; but there may exist many components that are much better conditioned. Unless the perturbations are restricted, no norm-based relative error bound can predict the presence of well-conditioned components, so these component-wise condition numbers are essential. These results are further extended in \cite{zeng2019sensitivity} and shown that 
 the sensitivity of a singular linear system
$A \x = \b$ is measured by $\| A\|_ 2 \| A^\dagger \|_ 2$.  

We make an interesting observation regarding the
sensitivity of solutions of  $A \x = \b$. Consider the linear system $A\mathbf{x}=\mathbf{b}$. Let $\tilde A$  be  a perturbation of $A$ by adding a random column vector to $\F_i$, that is $\tilde A=\left[
 \begin{array}{c|c|c|c|c|c|c}
\F_1& \cdots& \F_{i-1}& \F_i+\mathbf{E}&  \F_{i+1}&\cdots &\F_n
 \end{array}
 \right].
 $
 Consider the solution $\mathbf{\tilde x}$ to the  least squares  problem $\tilde A \mathbf{\tilde x}=\mathbf{b}$. We observe that if $\F_j$ is a columns that is not in the same cluster as $\F_i$ then 
 $x_j=\tilde x_j$.  As we can see from Example \ref{ex3}, columns that are in the same cluster are intertwined with each other and isolated from other clusters; so perturbations to a column will affect only the components of solutions corresponding to  columns in the same cluster. 
We shall prove this in Theorem \ref{pert} for  rank-deficient matrices. To do so, we first realize that $\tilde A$ can be viewed as a rank-1 update of $A$. Then we use Meyer's result   \cite{meyer1973generalized} that provides the  pseudo-inverse of rank-1 updates.
These pseudo-inverses  can be written in terms of $A^\dagger$ and a sum other matrix products that involve rows of the signature matrix $S$ of $A$. 
We conclude by  briefly explaining how this latter result can be applied in (un)supervised learning.

\section{Main results}
Let $A$ be an $m\times n$ matrix of rank $\rho$ and consider the SVD of $A$ as   $A=U\Sigma V^T$, where $U_{m\times m}$ and $V_{n\times n}$ are orthogonal matrices and  
$\Sigma=\text{diag}(\sigma_1, \ldots, \sigma_{\rho}, 0, \ldots, 0 )$ is an $m\times n$ diagonal matrix. Also, recall that the Moore-Penrose inverse of $A$ is the $n\times m$ matrix $A^{\dagger}=VS^{-1}U^T$, where $S^{-1} = \text{diag}(\sigma_1^{-1}, \ldots, \sigma_{\rho}^{-1}, 0, \ldots, 0 )$. Throughout, we shall always use 2-norm of a vector or matrix.

We denote column $j$ of $V$ by $\mathbf{v}_j$ and row $j$ of $V$ by $\mathbf{v}^j$.   Furthermore, we partition $\mathbf{v}^j$ as 
$\mathbf{v}^j=
\left[
\begin{array}{c|c}
\mathbf{v}^{j,1} & \mathbf{v}^{j,2}
\end{array}
\right]
$, 
where $\mathbf{v}^{j,1}$ consists of the first $\rho$ entries of $\mathbf{v}^{j}$ and 
$\mathbf{v}^{j,2}$ is the remaining $n-\rho$ entries. Note  that $A\mathbf{v}_j=0$, for all $\rho+1\leq j\leq n$, and moreover $\ker (A)$  is spanned by all $\mathbf{v}_{\rho+1}, \ldots, \mathbf{v}_{n} $. We denote by $\mathbf{F_j}$  the $j$-th column of $A$.

Let $\bar V$ be the matrix consisting of columns $\rho+1, \ldots, n$ of $V$, that is $\bar V=
\left[
\begin{array}{c|c|c}
\mathbf{v}_{\rho+1} & \cdots & \mathbf{v}_{n}
\end{array}
\right].
$
Let $P=\bar V \bar{V}^T$. Note that 
$P\mathbf{w}=\mathbf{w}$, for every $\mathbf{w}\in \mathcal{N}(A)$, where $\mathcal{N}(A)$ is the null space   of $A$. Indeed,   $P$ is the orthogonal projection onto $\mathcal{N}(A)$, that is range of $P$ is  $\mathcal{N}(A)$, $P^2=P$ and $P^T=P$. We also let $S=I-A^\dagger A$. The matrices $S$ and $P$ are closely related as the following result shows. We denote by $\mathbf{e}_i$ the $i$-th standard column vector.

\begin{lemma}\label{SP}
	We have $S=P$. 
\end{lemma}
\begin{proof} 
	 Note that 
	\begin{align*}
		S_{i,j}=	\e_i^T(I-A^\dagger A) \e_j&=\e_i^T\e_j-\e_i^TVS^{-1}U^TUSV^T\mathbf{e}_j=\e_i^T\e_j-\mathbf{e}_i^TVS^{-1}SV^T\mathbf{e}_j\\
		&=\e_i^T\e_j-\mathbf{v}^i
		\left[
		\begin{array}{c|c}
			I_{\rho} & 0 \\
			\hline
			0 & 0
		\end{array}
		\right]
(\mathbf{v}^j)^T
		=\e_i^T\e_j-\langle \mathbf{v}^{i,1}, \mathbf{v}^{j,1}\rangle.
	\end{align*}
	Now, if $i\neq j$ then  $\mathbf{v}^{i}$ and  $\mathbf{v}^{j}$ are orthogonal and so we have 
	$S_{i,j}=-\langle \mathbf{v}^{i,1}, \mathbf{v}^{j,1}\rangle=\langle \mathbf{v}^{i,2}, \mathbf{v}^{j,2}\rangle=P_{i,j}.$
	  Similarly, 
$S_{i,i}=1-\langle \mathbf{v}^{i,1}, \mathbf{v}^{i,1}\rangle=\langle \mathbf{v}^{i,2}, \mathbf{v}^{i,2}\rangle=P_{i,i}$.
\end{proof}

 Even though, $S$ and $P$ are the same, the computational complexity of computing of $S$ and $P$ might be different. For to compute $P$ we just need the right singular vectors of the symmetric matrix $A^TA$. On the other hand, if $A$ is full row rank then we know 
 $A^{\dagger}=A^T(AA^T)^{-1}$. So in case $A$ has full row-rank, the complexity of computing $S$ is the same as complexity of matrix inversion. 
 
 \begin{theorem}\label{indp-thm2} The following are equivallent for a matrix $A$.
 	\begin{enumerate}
 		\item 	The  column $\mathbf{F}_j$ of $A$ is independent of the rest of columns of $A$;
 		\item 	$\mathbf{v}^{j,2}=0$;
 		\item $\mathbf{e}_j^T(I-A^\dagger A)=0$;
 		\item $P_{j,j}=0$.
 	\end{enumerate}
 \end{theorem}
 \begin{proof}
 	Note  that $A\mathbf{v}_i=0$, for all $\rho+1\leq i\leq n$.
 	Let $k$ be in the range $\rho+1\leq k\leq n$. Note that  $A\mathbf{v}_{k}=0$ yields a dependence relation between the columns of $A$. So if  $\mathbf{F}_j$ is independent of the rest of columns of $A$, we deduce that the entry in the $j$-th position of $\mathbf{v}_{k}$ must be zero, that is $v_{j, k}=0$. 
 	So the $j$-th row of $V$ is of the form 
 	$\mathbf{v}^j=[v_{j, 1} \, \cdots\, v_{j, \rho}\, 0\, \cdots \,0 ]$. Hence, $\mathbf{v}^{j,2}=0$. This proves $(1) \Rightarrow (2)$. Now suppose that $\mathbf{v}^{j,2}=0$. So, we have 
 	\begin{align*}
 	\mathbf{v}^j
 	\left[
 	\begin{array}{c|c}
 	I_{\rho} & 0 \\
 	\hline
 	0 & 0
 	\end{array}
 	\right]
 	=	\mathbf{v}^j.
 	\end{align*}
 	We have
 	\begin{align*}
 	\mathbf{e}_j^TA^\dagger A&=\mathbf{e}_j^TVS^{-1}U^TUSV^T=\mathbf{e}_j^TVS^{-1}SV^T\\
 	&=\mathbf{v}^j
 	\left[
 	\begin{array}{c|c}
 	I_{\rho} & 0 \\
 	\hline
 	0 & 0
 	\end{array}
 	\right]
 	V^T
 	=\mathbf{v}^j
 	V^T
 	=\mathbf{e}_j^TI.
 	\end{align*}
 	Hence, $\mathbf{e}_j^T(I-A^\dagger A)=0$. So,  $(2) \Rightarrow (3)$.  The implication $(3) \Rightarrow (4)$ is rather obvious because $\mathbf{e}_j^T(I-A^\dagger A)=0$ means that  the entire $j$-th row of $I-A^\dagger A$ is zero. So, by Lemma \ref{SP},  $P_{j,j}=0$. Finally, suppose 
 	$P_{j,j}=0$. Note that  a dependence relation between $\mathbf{F}_j$ and the other columns, yields a vector $\mathbf{z}$ whose $j$-th position is non-zero and $A\mathbf{z}=0$. So, $\mathbf{z}$ is in the $\ker(A)$ and can be expressed in terms of $\mathbf{v}_{\rho+1}, \ldots, \mathbf{v}_n$. So,  the $j$-th component of at least one of the $\mathbf{v}_{\rho+1}, \ldots, \mathbf{v}_n$ must be non-zero. 
 	Hence,  $\mathbf{v}^{j,2}\neq 0$. But then $P_{j,j}=\langle \ve^{j,2}, \ve^{j,2}\rangle\neq 0$, which is a contradiction. Hence, $\F_j$ is independent of the rest of columns of $A$. So, $(4) \Rightarrow (1)$, as required.
 \end{proof}

\begin{corollary}\label{Pii}
	Suppose that $\F_1, \ldots,\F_t$ are linearly dependent and independent of the rest of the $\F_k$'s. Then 
 $P_{i,i}\neq 0$, for every $1\leq i\leq t$.
\end{corollary}

\begin{ex}\label{ex2}
Consider a $50\times 40$ synthetic matrix $A$ with the only relations between columns of $A$  as follows:
\begin{align*}
\begin{array}{ccc}
-\F_1+2\F_5+\F_6=0,&  \F_1-\F_2-3\F_5+\F_6=0,\\
	-\F_3+\F_5-3\F_6=0,&\quad \F_3-\F_4+2\F_5+4\F_6=0,\\
-\F_7+\F_9-5\F_{10}=0,&  	-\F_8+5\F_9+\F_{10}=0. 
\end{array}
\end{align*}
We note that  $A$ is randomly generated and the only constrain on $A$ is the set of dependent relations given above. Nevertheless, $S$ is independent of $A$ and it captures the dependencies between columns of $A$.

As expected, by Theorem \ref{indp-thm2}, columns 11-40 of $S$ are entirely zero because those columns  are independent from the rest of columns of $A$. 
The signature matrix $S$ (rounded up to two decimals) for $A$ is:
\footnotesize
\setlength{\arraycolsep}{2.5pt} 
\medmuskip = 1mu
\begin{equation*}
\resizebox{.9\linewidth}{!}{%
	$
\left(\begin{array}{ccccccccccccc} 
	0.69 & 0 & 0.06 & -0.44 & -0.12 & -0.06 & 0 & 0 & 0 & 0 & 0  & \cdots & 0\\
	 0 & 0.69 & 0.44 & 0.06 & 0.06 & -0.12 & 0 & 0 & 0 & 0 & 0& \cdots & 0\\
   0.06 & 0.44 & 0.38 & 0 & -0.06 & 0.19 & 0 & 0 & 0 & 0 & 0& \cdots & 0\\
  -0.44 & 0.06 & 0 & 0.38 & -0.19 & -0.06 & 0 & 0 & 0 & 0 & 0 & \cdots & 0\\ 
    -0.12 & 0.06 & -0.06 & -0.19 & 0.94 & 0 & 0 & 0 & 0 & 0 & 0 & \cdots & 0\\
     -0.06 & -0.12 & 0.19 & -0.06 & 0 & 0.94 & 0 & 0 & 0 & 0 & 0 &\cdots &  0\\
      0 & 0 & 0 & 0 & 0 & 0 & 0.04 & 0 & -0.04 & 0.19 & 0 & \cdots & 0\\
      0 & 0 & 0 & 0 & 0 & 0 & 0 & 0.04 & -0.19 & -0.04 & 0 & \cdots & 0\\
       0 & 0 & 0 & 0 & 0 & 0 & -0.04 & -0.19 & 0.96 & 0 & 0 & \cdots & 0\\
        0 & 0 & 0 & 0 & 0 & 0 & 0.19 & -0.04 & 0 & 0.96 & 0 & \cdots & 0\\
         0 & 0 & 0 & 0 & 0 & 0 & 0 & 0 & 0 & 0 & 0 & \cdots & 0\\ 
         \vdots & \vdots &\vdots &\vdots &\vdots &\vdots &\vdots &\vdots &\vdots &\vdots &\vdots &\cdots &\vdots     \\
         0 & 0 & 0 & 0 & 0 & 0 & 0 & 0 & 0 & 0 & 0 & \cdots & 0   \\
      \end{array}\right)            
$}
\end{equation*}

 \normalsize  
 
 Aside from columns 11-40 of $A$ that each form a cluster with a single element, there are two major clusters with one consisting of columns $\F_1,\ldots, \F_6$ and the other consisting of columns $\F_7,\ldots, \F_{10}$. 
 The graph associated to $A$ is pictured below demonstrating the two clusters.
 We note however, that these two clusters are not fully connected, for example $S_{3,4}=0$ meaning that there is no edge between $\F_3$ and $\F_4$. The idea is that columns that correlate with each other form a (connected) cluster.
 \smallskip

 	\begin{figure}[H]
 		 \begin{center}
 \begin{tikzpicture}[scale=3, vertex/.style={draw,circle,fill=tempcolor}, arc/.style={draw,thick,-, green}]
 \foreach [count=\i] \coord in {(0.809,0.588),(0.309,0.951),(-0.309,0.951),(-0.809,0.588),(-1.,0.),(-0.809,-0.588),(-0.309,-0.951),(0.309,-0.951),(0.809,-0.588),(1.,0.)}{
 	\node[vertex] (p\i) at \coord {\i};
 }
 
 \foreach [count=\r] \row in {
 	{0,0,1,1,1,1,0, 0,0,0},
 	{0,0,1,1,1,1,0,0,0,0},{1,1,0,0,1,1,0,0,0,0},{1,1,0,0,1,1,0,0,0,0},{1,1,1,1,0,0,0,0,0,0},{1,1,1,1,0,0,0,0,0,0},
 	{0,0,0,0,0,0, 0, 0, 1,1},{0,0,0,0,0,0,0,0,1,1},{0,0,0,0,0,0,1,1,0,0},{0,0,0,0,0,0,1,1,0,0}}{
 	\foreach [count=\c] \cell in \row{
 		\ifnum\cell=1
 		\draw[arc] (p\r) edge (p\c);
 		\fi
 	}
 }
 \end{tikzpicture}
 \caption{The graph associated to matrix $A$ demonstrating the two clusters.}
  \end{center}
\end{figure}
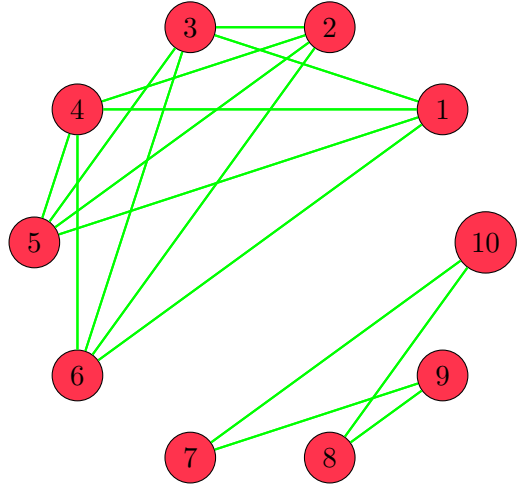
 
 \smallskip
\end{ex}

What determines whether two columns $\F_i$ and $\F_j$ are part of the same cluster is the existence of a linearly dependent 
set $Y$ consisting of some columns of $A$ so that  $\F_i$ and $\F_j$ are in $Y$. In Theorem \ref{main1}, we prove that
if $\F_i$ and $\F_j$ are in different clusters then $P_{i,j}=0$. The converse of this, however, does not hold as one might hope so. That is within the same cluster, there might be $\F_i$ and $\F_j$ such that $P_{i,j}=0$ as can be seen in Example \ref{ex2}. 

We associate a graph $G(A)$ to $A$ whose vertices are the columns of $A$ and we say $\F_i$ and $\F_j$ are connected if  $P_{i,j}\neq 0$. In Theorem \ref{graph}, we prove that each connected component of this graph correspond to a linearly dependent   subset of columns of $A$. 

After identifying clusters (connected components) of $G(A)$ we can even determine the set of minimal relations between columns in each cluster. We shall explain this process for  the matrix $A$ given in Example \ref{ex2}.
\begin{ex}\label{minimal-rel}
	Let $A$ be the matrix in Example \ref{ex2}.  Suppose that we have already identified that columns $\F_1, \ldots,\F_6$ are in the same cluster. We note that rank of $A$ is $\rho=34$. Hence, $A\ve_k=0$, for every $35\leq k\leq 40$.
	Since $A\ve_k=0$ yields a dependence relation between columns of $A$ and $\F_1, \ldots,\F_6$ are independent from the rest of the columns, we deduce that  $A\bar  \ve_k=0$, where $\bar  \ve_k$ consists of the first 6 entries of $\ve_k$. Then we form the matrix 
	$
	M=
	\left[
	\begin{array}{c|c|c}
	\bar{\ve}_{35} & \cdots & \bar{\ve}_{40}
	\end{array}
	\right]
	$. Since any linear combination of columns of $M$ provides a dependence relation between $\F_1, \ldots,\F_6$, we can use elementary (column) operations to transform $M$ into the matrix $\bar C$: 
	
	\begin{align*}
		\bar C=
	\left(\begin{array}{cccccc} 
		-1.0 & 0 & 0 & 0& 0 & 0\\ 
		0 & -1.0 & 0 & 0& 0 & 0\\ 
		0 & 0 & -1.0 & 0& 0 & 0\\ 
		0 & 0 & 0 & -1.0& 0 & 0\\ 
		2.0 & -1.0 & 1.0 & 3.0& 0 & 0\\
		 1.0 & 2.0 & -3.0 & 1.0 & 0 & 0
	 \end{array}\right).
 \end{align*}

	 Then  $\left[
	 \begin{array}{c|c|c}
	 \F_1 &\cdots  &\F_6 
	 \end{array}
	 \right]
	 \bar C=0$; in other words non-zero columns of $\bar C$ give us the  minimal relations between $\F_1, \ldots,\F_6$.
	
\end{ex}

\begin{lemma}\label{cor-inv} 	Let $Z$ be an $m\times n$ matrix.
	Every block matrix of the form 
	$
	B=
	\left[
	\begin{array}{c|c}
	-I_n & Z^T \\
	\hline
	Z &I_{m}
	\end{array}
	\right]
	$ is invertible.
\end{lemma}
\begin{proof} First we show that $I+ZZ^T$ is invertible. 
	Consider the SVD of $Z$ as $Z=U\Sigma V^T$. Let $\sigma_1\geq  \cdots\geq  \sigma_r\geq 0$ be the singular values of  $Z$, where $r=\min\{m,n\}$.  So, $\Sigma$ is an $m\times n$ diagonal matrix,
	with diagonal elements $\sigma_i$, for all $i=1\cdots r$. Note that 
	$I+ZZ^T=U(I+\Sigma\Sigma^T)U^T$. Since $I+\Sigma\Sigma^T$ is a diagonal matrix with $1+\sigma_i^2$ on the diagonal, $I+\Sigma\Sigma^T$ is invertible and so is  $I+ZZ^T$. 
It is easy to verify that 
	\begin{align*}
	B^{-1}
	=\left[
	\begin{array}{cc}
	-(I+Z^TZ)^{-1}& (I+Z^TZ)^{-1} Z^T\\
	Z(I+Z^TZ)^{-1} & (I+ZZ^T)^{-1}
	\end{array}
	\right].
	\end{align*}

	\end{proof}

\begin{ex}
	Let $A$ be the matrix in Example \ref{ex2}. Then, $P_{i,j}=0$ for all $1\leq i\leq 6$ and $7\leq j\leq n$.
\end{ex}
\begin{proof}
	Let $k$ be in the range $35\leq k\leq 40$. Since $A\F_k=0$, we have
$v_{1,k}\F_1+v_{2,k}\F_2+v_{3,k}\F_3+v_{4,k}\F_4+v_{5,k}\F_5=0$. Substituting in terms of   $\F_5$ and $\F_6$ using the 
matrix $\bar C$ from Example \ref{minimal-rel}, we get 
\begin{align*}
v_{1,k}(2\F_5+\F_6)+v_{2,k}(-\F_5+2\F_6)+v_{3,k}(\F_5-3\F_6)+v_{4,k}(3\F_5+\F_6)+v_{5,k}\F_5+v_{6,k}\F_6=0
\end{align*}
We deduce that
\begin{align*}
2v_{1,k}-v_{2,k}+v_{3,k}+3v_{4,k}+v_{5,k}&=0,\\
v_{1,k}+2v_{2,k}-3v_{3,k}+v_{4,k}+v_{6,k}&=0.
\end{align*}
Since the above equations hold for every $k$ in the range $\rho+1\leq k\leq n$, we deduce that 
\begin{align*}
2\mathbf v^{1,2}-\ve^{2,2}+\mathbf{v}^{3,2}+3\mathbf{v}^{4,2}+\mathbf{v}^{5,2}&=0,\\
\mathbf v^{1,2}+2\ve^{2,2}-3\mathbf{v}^{3,2}+\mathbf{v}^{4,2}+\mathbf{v}^{6,2}&=0.
\end{align*}
Let $j$ be in the range  $7\leq j\leq n$. Then  taking the dot product with $\mathbf{v}^{j,2}$ yields 
\begin{align}\label{1st}
2P_{1,j}-P_{2,j}+P_{3,j}+3P_{4,j}+P_{5,j}&=0,\nonumber \\
P_{1,j}+2P_{2,j}-3P_{3,j}+P_{4,j}+P_{6,j}&=0.
\end{align}
Let 
$
C=\left[
\begin{array}{c|c}
\bar C &0\\
\hline
0 &0
\end{array}
\right]
$ be an $n\times n$ matrix. 
Let $\mathbf{c}_1, \ldots, \mathbf{c}_n$ be the columns of $C$ and denote by $\mathbf{p}^j$ the $j$-th row of $P$. Since $P\mathbf{c}_i=\mathbf{c}_i$, we deduce that 
$\mathbf{p}^j\mathbf{c}_i=\mathbf{c}_{i,j}=0$, since $j\geq 7$. Hence, 
\begin{align}\label{2ed}
-P_{1,j}+2P_{5,j}+P_{6,j}&=0,\nonumber \\
-P_{2,j}-P_{5,j}+2P_{6,j}&=0,\nonumber \\
-P_{3,j}+P_{5,j}-3P_{6,j}&=0,\nonumber \\
-P_{4,j}+3P_{5,j}+P_{6,j}&=0.
\end{align}
Putting together the Equations \eqref{1st} and \eqref{2ed}, we deduce that 
$$
B\begin{bmatrix}
P_{1,j} & P_{2,j} & P_{3,j} & P_{4,j} & P_{5,j} & P_{6,j}
\end{bmatrix}^T
=0,
$$
where
\begin{align*}
B=
\begin{bmatrix}
-1 & 0& 0& 0  & 2 & 1\\
0&-1 &  0& 0  & -1 & 2\\
0& 0&-1 &   0  & 1 & -3\\
0  & 0& 0&-1 &    3 & 1\\
2  & -1& 1&3&    1 & 0\\
1  & 2& -3&1&    0 & 1\\
\end{bmatrix}
=\left[
\begin{array}{c|c}
-I_4 &Z^T\\
\hline
Z &I_2
\end{array}
\right], \quad
Z=
\left[
\begin{array}{cccc}
	2  & -1& 1&3\\
	1  & 2& -3&1
\end{array}
\right].
\end{align*}

Since, by Lemma \ref{cor-inv}, $B$ 
is invertible, we deduce that  $P_{1,j} = \cdots= P_{6,j}=0$.

\end{proof}

For the rest of this section, we assume that $\tau=\{\F_1, \ldots,\F_t\}$ is a cluster, that is $\F_1, \ldots,\F_t$ are linearly dependent and independent of the rest of the $\F_k$, where $k\geq t+1$.
This means there are linear equations that yield  dependencies  between of  columns  $\tau$. 
Suppose that the dimension of the subspace spanned by $\F_1, \ldots,\F_t$  is $t-r$, for some $r\geq 1$.
Without loss of generality we assume that $\F_{r+1}, \ldots,\F_t$ are independent of each other.

\begin{theorem}\label{main1}
We have $P_{i,j}=0$, for every $1\leq j\leq t$ and every $i\geq t+1$.
\end{theorem}
\begin{proof}
Note that $A\ve_k=0$ for each $k$ in the range $\rho+1\leq k\leq n$. Let $\bar{\ve}_k$ be the first $t$ entries of $\ve_k$. Consider the $t\times (n-\rho)$ matrix 
$
M=
\left[
\begin{array}{c|c|c}
\bar{\ve}_{\rho+1} & \cdots & \bar{\ve}_n
\end{array}
\right]
$.
Note that any linear dependence between $\F_1, \ldots,\F_t$ yields a vector that lies in the column space of $M$ and vice versa. We use elementary column operations to transform $M$ to a matrix of the form
\begin{align}\label{matrixC}
\bar C=
\begin{bmatrix}
-1& 0 & 0&  0 & 0 &0 &\cdots& 0  \\
0 & -1 & 0& 0  &0 &0 & \cdots& 0 \\
\vdots & \vdots & \vdots& \vdots & \vdots & \vdots& \vdots & \vdots \\
0 &0 & \cdots & 0 & -1& 0 &\cdots& 0\\
c_{r+1, 1} & c_{r+1, 2} & \cdots & \cdots & c_{r+1, r}& 0 &\cdots& 0\\
\vdots & \vdots & \vdots& \vdots & \vdots& \vdots &  \vdots & \vdots \\
c_{t,1} & c_{t,2} & \cdots &\cdots& c_{t,r} & 0& \cdots& 0
\end{bmatrix}
=
\left[
	\begin{array}{c|c}
-I_r & 0 \\
\hline
Z& 0
\end{array}
\right].
\end{align}

Let us denote by $\bar{\mathbf{c}}_j$ the $j$-th column of $\bar C$. We also let 
$\mathbf{c}_j=
\begin{bmatrix}
\bar{\mathbf{c}}_j & 0& \cdots & 0
\end{bmatrix}^T\in \mathbb{R}^n$ and set $
C=
\left[
\begin{array}{c|c|c}
\mathbf{c}_1 & \cdots & \mathbf{c}_{n-\rho}
\end{array}
\right]
$.
Note that $\mathbf{c}_j\in \ker(A)$. Furthermore,   any linear dependence between $\F_1, \ldots,\F_t$ yields a vector that lies in the subspace spanned by  $\mathbf{c}_1,\cdots, \mathbf{c}_r$ and vice versa.   Let $k$  be in the range $\rho+1\leq k\leq n$. Since $A\ve_k=0$, we get 
$v_{1,k}\F_1+\cdots +v_{t,k}\F_t=0$. Now substituting for $\F_1, \ldots,\F_r$ in terms of 
$\F_{r+1}, \ldots,\F_t$ and using the matrix $ C$ implies the following equations:

\begin{align}\label{csvs}
\begin{cases}
c_{t,1}v_{1,k}+c_{t,2}v_{2,k}+\cdots + c_{t,r}v_{r,k}+ v_{t,k}=0\\
c_{t-1,1}v_{1,k}+c_{t-1,2}v_{2,k}+\cdots + c_{t-1,r}v_{r,k}+ v_{t-1,k}=0\\
\vdots\\
c_{r+1,1}v_{1,k}+c_{r+1,2}v_{2,k}+\cdots + c_{r+1,r}v_{r,k}+ v_{r+1,k}=0.
\end{cases}
\end{align}

Since Equations \eqref{csvs} hold for every $k$ in the range $\rho+1\leq k\leq n$, we deduce that 
\begin{align}\label{cs-vs}
\begin{cases}
c_{t,1}\ve^{1,2}+c_{t,1}\ve^{2,2}+\cdots + c_{t,r}\ve^{r,2}+ \ve^{t,2}=0\\
\vdots\\
c_{k,1}\ve^{1,2}+c_{k,2}\ve^{2,2}+\cdots + c_{k,r}\ve^{r,2}+ \ve^{k,2}=0\\
\vdots\\
	c_{r+1,1}\ve^{1,2}+c_{r+1,2}\ve^{2,2}+\cdots + c_{r+1,r}\ve^{r,2}+ \ve^{r+1,2}=0.
\end{cases}
\end{align}

Let $i$ be in the range $1\leq i\leq n$. Multiplying each of the equations in \eqref{cs-vs} with $\ve^{i,2}$ yields the following:
\begin{align*}
\begin{cases}
c_{t,1}P_{1,i}+\cdots+c_{t,j}P_{j,i}+\cdots + c_{t,r}P_{r,i}+ P_{t,i}=0\\
\vdots\\
c_{k,1}P_{1,i}+\cdots + c_{k,j}P_{j,i}+\cdots + c_{k,r}P_{r,i}+ P_{k,i}=0\\
\vdots\\
c_{r+1,1}P_{1,i}+\cdots+c_{r+1,j}P_{j,i}+\cdots + c_{r+1,r}P_{r,i}+ P_{r+1,i}=0.
\end{cases}
\end{align*}

Writing the above equations in the matrix form, we get
\begin{align}
\begin{bmatrix}\label{cs-Ps}
	c_{r+1, 1}  & \cdots & c_{r+1, r}& 1 & 0& \cdots&  0\\
	\vdots  & \vdots& \vdots & \vdots& \vdots & \vdots &\vdots \\ 
	c_{k, 1}  & \cdots & c_{k, r}& 0 &\cdots& 1 & 0\\
	c_{t,1} & \cdots & c_{t,r}& 0 &\cdots& 0 & 1\\
\end{bmatrix}
\begin{bmatrix}
P_{i, 1}\\
P_{i, 2}\\
\vdots\\
P_{i, t}
\end{bmatrix}
=
\left[
\begin{array}{c|c}
	Z& I
\end{array}
\right]
\begin{bmatrix}
	P_{i, 1}\\
	P_{i, 2}\\
	\vdots\\
	P_{i, t}
\end{bmatrix}
=
0.\end{align}

Furthermore, since  $P\mathbf{c}_j=\mathbf{c}_j$, 
we deduce that $\mathbf{p}^i \mathbf{c}_j=c_{i, j}$, for all $i$ in the range $1\leq i\leq n$ and $j$ in the range $1\leq j\leq r$. We   deduce that
\begin{align*}
\begin{cases}
	-P_{i, 1}+c_{r+1,1} P_{i, r+1}+\cdots + c_{t,1}P_{i, t}&= c_{i, 1}\\
	\vdots\\
	-P_{i,j}+c_{r+1,j} P_{i, r+1}+\cdots + c_{t,j}P_{i, t,}&= c_{i, j}\\
	\vdots&\vdots \\
	-P_{i, r}+c_{r+1,r} P_{i, r+1}+\cdots + c_{t,r}P_{i, t}&= c_{i, r}.
\end{cases}
\end{align*}
Hence, 
\begin{align}\label{cs-Ps-third}
\left[
\begin{array}{c|c}
	-I &Z^T
\end{array}
\right]
\begin{bmatrix}
	P_{i, 1}\\
	P_{i, 2}\\
	\vdots\\
	P_{i, t}
\end{bmatrix}
=
\begin{bmatrix}
	c_{i, 1}\\
	c_{i, 2}\\
	\vdots\\
	c_{i, r }
\end{bmatrix}.
\end{align}

Note that the Equations \eqref{cs-Ps} and \eqref{cs-Ps-third} can be put together and 
written in the matrix form as 
\begin{align}\label{BPc}
B
\begin{bmatrix}
P_{i,1} \\
\vdots \\
P_{i,t}
\end{bmatrix}
=\begin{bmatrix}
	c_{i, 1}\\
	\vdots\\
	c_{i, r }\\
	0\\
	\vdots\\
	0
\end{bmatrix},
\end{align}
where 
$$
B=
\left[
\begin{array}{c|c}
-I_r & Z^T \\
\hline
Z &I_{t-r}
\end{array}
\right],
$$ 
Note that, by Lemma \ref{cor-inv}, $B$ is invertible.
Also if  $i$ is in the range $t+1\leq i\leq n$, we have  $\mathbf{c}^i=0$.
Hence, $P_{i,j}=0$, for all $ t+1\leq i\leq n$ and $1\leq j\leq t$.
\end{proof}

We can determine whether   columns $\F_i$ and $\F_j$ are in different clusters  using Theorem \ref{main1} by checking whether $P_{i,j}$ is zero or not. However, as  shown in Example \ref{ex2}, the converse is not true. In other words, within a cluster we could have columns $\F_i$ and $\F_j$ such that $P_{i,j}=0$. The next lemma sheds some lights on the structure of $P_{i,j}$'s for columns within a cluster.

\begin{lemma}\label{Pij}
	For every $i$ in the range $1\leq i\leq t$ there exists $1\leq j\neq i\leq t$ such that  $P_{i,j}\neq 0$.
\end{lemma}
\begin{proof}
 If $1\leq i\leq r$ then  there is  $k$ in the range 
	$r+1\leq k\leq t$ such that $c_{k,i}\neq 0$. Now 
	consider the following equation from \eqref{cs-Ps} 
	\begin{align*}
		c_{k,1}P_{1,i}+\cdots+ c_{k,i}P_{i,i}+\cdots + c_{k,r}P_{r,i}+ P_{k,i}=0.
	\end{align*}
Since $c_{k,i}P_{i,i}\neq 0$, there exists $j\neq i$ such that $P_{i,j}\neq0$.
	On the other hand if $r+1\leq i\leq t$ then  
	consider the following equation from \eqref{cs-Ps} 
	\begin{align*}
		c_{i,1}P_{1,i}+\cdots+ c_{i, j}P_{j,i}+\cdots + c_{i,r}P_{r,i}+ P_{i,i}=0.
	\end{align*}
	Since $P_{i,i}\neq 0$, we deduce that there exists $j$ such that $c_{i,j}P_{j,i}\neq 0$, hence $P_{i,j}\neq 0$.
\end{proof}

\begin{lemma}\label{Pii1}
	For every $i$ in the range $1\leq i\leq t$, we have $P_{i,i}<1$.
	\end{lemma}
\begin{proof}
	We know that $P_{i,i}=\langle \ve^{i,2}, \ve^{i,2}\rangle=\langle \ve^{i}, \ve^{i}\rangle-\langle \ve^{i,1}, \ve^{i,1}\rangle=1-\langle \ve^{i,1}, \ve^{i,1}\rangle$. Suppose that
	 $P_{i,i}=1$. Then we must have  $\langle \ve^{i,1}, \ve^{i,1}\rangle=0$ which in turn implies that $\ve^{i,1}=0$. But then 
	 \begin{align*}
	 0=	\langle \ve^{i}, \ve^{j}\rangle=\langle\ve^{i,1}, \ve^{j,1}\rangle+
	 \langle\ve^{i,2}, \ve^{j,2}\rangle= P_{i,j},
	 	\end{align*}
 	for all $j\neq i$. This contradicts Lemma \ref{Pij}. Hence, $P_{i,i}<1$.

	\end{proof}

We might wonder what can be said about the magnitude of each $P_{i,j}$. The next result provides an upper bound even though tighter bounds might be possible.
\begin{lemma}\label{Pij2}
	For every $i,j$ in the range $1\leq i,j\leq n$, we have  $| P_{i,j}| \leq 2$.
\end{lemma}
\begin{proof}
	Consider the matrix  $S=I-A^\dagger A$. Note that $\parallel S \parallel \leq \parallel I \parallel + \parallel  A^\dagger A\parallel\leq 2$. Hence, norm of each row or column  of $S$ is at most 2.
	If $i\neq j$ then, by Lemma \ref{indep}, we have $|P_{i,j}|=|S_{i,j}|\leq \mathbf{s}^i\leq 2$. If $i=j$ then the result follows from Lemma \ref{Pii1}.
\end{proof}

Now we define a graph $G$ whose vertices consists of $\F_1, \ldots, \F_n$ and we define an edge between $\F_i$ and $\F_j$ if and only if $P_{i, j}\neq 0$. Let us consider the subgraph of $G$ consisting of $\F_1, \ldots, \F_n$ and the corresponding nodes. Our goal is to show in Theorem \ref{graph} that this subgraph is connected.

	Consider the matrix $A$ in Example \ref{ex2}  and let $\Gamma$ be subgraph consisting of nodes $\F_1,\ldots, \F_6$. If we consider the bipartite graph  where the set of nodes is partitioned to 
	subsets $\{\F_1, \ldots, \F_4\}$ and $\{\F_5, \F_6\}$, then we note that $\Gamma$ is a completed version of this  bipartite graph.
	
The idea in the general case is to show there is connectivity between nodes in the sets $\{\F_1,\ldots, \F_r\}$ and $\{\F_{r+1}, \ldots, \F_t\}$.
The next lemma is an step in that direction.

\begin{lemma}\label{Pkj} The following statements hold.
\begin{enumerate}
	\item 	For every $k$ in the range $r+1\leq k\leq t$ there exists $1\leq j\leq r$ such that  $P_{k,j}\neq 0$.
	\item For every  $i$ in the range $1\leq i\leq r$ there exists $r+1\leq j\leq t$ such that  $P_{i,j}\neq 0$.
\end{enumerate}
\end{lemma}
\begin{proof}
By Equation \ref{BPc}, we have

\begin{align*}
B
\begin{bmatrix}
P_{i,1} \\
\vdots \\
P_{i,t}
\end{bmatrix}
=\begin{bmatrix}
c_{i, 1}\\
\vdots\\
c_{i, r }\\
0\\
\vdots\\
0
\end{bmatrix}
=(\bar \c^i)^T,
\end{align*}
where $1\leq i\leq  t$ and $\bar \c^i$ is the $i$-th row of $\bar C$. Hence,

  \begin{align}\label{BP}
B
\begin{bmatrix}
P_{1, 1}& \cdots & P_{t,1}\\
\vdots & \vdots &\vdots\\
P_{1,t}& \cdots & P_{t, t}\\
\end{bmatrix}=
B
\left[
\begin{array}{c|c}
	P_1&P_2\\
	\hline
	P_3 &P_4
\end{array}
\right]=
\left[
\begin{array}{c|c}
-I_r& Z^T\\
\hline
0 &0
\end{array}
\right]=\bar C^T,
\end{align}
where 
$
P_1=
\begin{bmatrix}
P_{1, 1}& \cdots & P_{r,1}\\
\vdots & \vdots &\vdots\\
P_{1,r}& \cdots & P_{r,r}\\
\end{bmatrix}
$ 
and 
$
P_2=
\begin{bmatrix}
P_{r+1, 1}& \cdots & P_{t,1}\\
\vdots & \vdots &\vdots\\
P_{r+1,r}& \cdots & P_{t,r}\\
\end{bmatrix}
$.  Thus,
$$
\left[
\begin{array}{c|c}
P_1&P_2\\
\hline
P_3 &P_4
\end{array}
\right]
=B^{-1}
\left[
\begin{array}{c|c}
-I_r& Z^T\\
\hline
0 &0
\end{array}
\right]
$$
where 
\begin{align*}
B^{-1}
=\left[
\begin{array}{cc}
-(I+Z^TZ)^{-1}& (I+Z^TZ)^{-1} Z^T\\
Z(I+Z^TZ)^{-1} & (I+ZZ^T)^{-1}
\end{array}
\right].
\end{align*}

We deduce that 
\begin{align}\label{P4-P2}
	 P_4=-ZP_2.
	\end{align}
	
Note that, by Corollary \ref{Pii},  $P_{k, k}\neq 0$, for every $r+1\leq k\leq t$. By \eqref{P4-P2}, we have $P_{k,k}=\sum_{j=1}^r c_{k,j}P_{j,k}$.  We deduce that there exists $j$ such that 
$c_{k, j}P_{k, j}\neq 0$. Hence, $P_{k, j}\neq 0$ and this proves Part (1).  To prove Part (2), 
consider the following equation  from \eqref{cs-Ps-third}  where we choose $i=j$:
\begin{align*}
	-P_{i,i}+c_{r+1,i} P_{i, r+1}+\cdots + c_{t,i}P_{i, t}&= c_{i, i}=-1.
\end{align*}
Since, by Lemma \ref{Pii1}, $P_{i,i}<1$, we deduce that there exists $r+1\leq j\leq t$ such that 
$c_{j,i}P_{i, j}\neq 0$. Hence, $P_{i, j}\neq 0$.
\end{proof}

Now we are ready to finish the proof that $\{\F_1, \ldots, \F_t\}$ is a maximally dependent subset if and only if the subgraph associated with this subset is connected.

\begin{theorem}\label{graph}
	The sub-graph of $G$ consisting of nodes $\F_1,\ldots, \F_t$  and corresponding edges is connected.
\end{theorem}
\begin{proof}
	Suppose to the contrary. Then  we can divide  the set $\{\F_1, \ldots, \F_t\}$ into two subsets $Y_1$ and $Y_2$ so that there is no edge between any  elements of $Y_1$ and any elements  of $Y_2$ . Without loss of generality  and in view of Lemma \ref{Pkj}, we can further assume that
	$Y_1=\{\F_1, \ldots, \F_p, \F_{r+1},\ldots,  \F_{q} \}$ and $Y_2=\{\F_{p+1}, \ldots, \F_r, \F_{q+1},\ldots,  \F_{t} \}$, where $1\leq p\leq r$ and $r+1\leq q\leq t$. 
	Recall from Equation \eqref{BP}  that 
	  \begin{align}
	B
	\left[
	\begin{array}{c|c}
	P1&P_2\\
	\hline
	P_3 &P_4
	\end{array}
	\right]=
	\left[
	\begin{array}{c|c}
	-I_r& Z^T\\
	\hline
	0 &0
	\end{array}
	\right],
	\end{align}
	where $P_1$ is a $r\times r$ matrix. 	From the structure of $Y_1$ and $Y_2$ it follows that $P_1$ and $P_2$ have  the form a $2\times 2$ block-diagonal matrix. Suppose that 
	\begin{align*}
	P_1=	\left[
		\begin{array}{c|c}
			Q_1&0\\
			\hline
			0 &Q_2
		\end{array}
		\right]
		\end{align*}
	Now, by Equation \eqref{BP}, we have $P_1=-(I+Z^TZ)^{-1}$. Hence, 
		\begin{align*}
		I+Z^TZ=	\left[
		\begin{array}{c|c}
			-Q_1^{-1}&0\\
			\hline
			0 &-Q_2^{-1}
		\end{array}
		\right].
	\end{align*} 

	Now consider the block-diagonal form of $P_2$:
\begin{align*}
	P_2=	\left[
	\begin{array}{c|c}
		Q_3&0\\
		\hline
		0 &Q_4
	\end{array}
	\right].
\end{align*}
Using  Equation \eqref{BP} again, we have $P_2=-(I+Z^TZ)^{-1}Z^T$.
Hence, 
\begin{align}\label{P2}
	Z^T{=} -(I+Z^TZ)P_2&=	
	\left[
	\begin{array}{c|c}
		-Q_1^{-1}&0\\
		\hline
		0 &-Q_2^{-1}
	\end{array}
	\right]
	\left[
	\begin{array}{c|c}
		Q_3&0\\
		\hline
		0 &Q_4
	\end{array}
	\right]\nonumber\\
	&=
		\left[
	\begin{array}{c|c}
		-Q_1^{-1}Q_3&0\\
		\hline
		0 &-Q_2^{-1}Q_4
	\end{array}
	\right].	
\end{align}

We deduce from the structure of matrix $\bar C$ in \eqref{matrixC} and Equation \eqref{P2}, that 
each of $\F_1, \ldots, \F_p$ can be written only in terms of  $\F_{r+1},\ldots,  \F_{q}$. 
In other words,  the two sets $Y_1$ and $Y_2$ are independent of each other which contradicts the fact that 
$\{\F_1, \ldots, \F_t\}$ is a linearly dependent set.
\end{proof}

\section{Applications}\label{applications}

Clustering is the process  of dividing  data points into a number of groups (clusters) such that data points in the same cluster share   similar properties.
We can view the data as a matrix by storing the data points as rows or columns of the matrix. 
We can distinguish independent columns using Theorem \ref{indp-thm2}. For the columns that correlate with other columns, we can use Theorem \ref{graph} to find the cluster that includes that column.
Algorithm \ref{alg-clustering} provides a procedure to find the clusters.

	\begin{algorithm}
		\label{alg-clustering}
		\caption{Find the clusters}
		\begin{algorithmic}
			\STATE \textbf{Input:} an $m\times n$ matrix $A$ 
            \STATE \textbf{Output:} Subsets of linearly dependent columns of $A$
			\STATE{  $S=I-A^{\dagger}A$}
			\STATE{ $\text{Idx}=\{i\mid S_{i,i}\neq 0, 1\leq i\leq n\}$}
			\STATE{Update $A$ with selecting those columns with indices in $\text{Idx}$}
			\STATE{Update $S$}
			\STATE $k=1$
			\WHILE{$\text{Idx}\neq \emptyset$}
		        	\STATE{ $\text{CL}[k]=\emptyset$}
					\STATE  $n=\text{Idx[last]}, i=\text{idx}[1],\text{CL}[k]=\{i\},j=i+1$
					\WHILE{$i\leq n$}
							\WHILE{$j\leq n$}
						        	\IF{$S_{i,j}\neq 0$}
				        		  \STATE{ $\text{CL}[k] = \text{CL}[k]\cup \{j\}$}
				        		  \STATE{$\text{Idx=idx}\setminus\{j\}$}
						        	\ENDIF			
						        	\STATE{ $j=\text{ind}[j+1]$}
							\ENDWHILE
							
							\STATE{ $i=\text{CL}[k][i+1], j=1$}
					\ENDWHILE
					\STATE $k=k+1$
			\ENDWHILE
			
		\end{algorithmic}
	\end{algorithm}

Supervised learning is a central problem in machine learning and data mining \cite{burkov2019hundred}. Let  $D=[A\mid \mathbf{b}]$ be a dataset, say a binary Cancer dataset,  where rows of $A$ are samples (patients), columns of $A$ are features (gene expressions) and    $\mathbf{b}$ is the class label that each of its entries are either 0 (noncancerous) or 1 (cancerous). The idea of supervised learning is to use part of samples as the training data to build a model that can be used to classify the remaining samples.

In large datasets that are a large number of features that are irrelevant, that is these features have negligible correlation with the class labels.  Irrelevant features act as noise in the data that not only they increase the  computational costs but in some cases divert the learning process toward weak model generation. The goal of feature selection methods is to select the most important and effective features \cite{guyon2003introduction}. So, feature selection can decrease the model complexity in the training phase while  retaining or  improving the classification accuracy.  

Since we know how to detect correlations, we consider the augmented matrix $D$ and form 
$S=I-D^{\dagger}D$. Rather than finding all clusters, we are interested to find only one cluster and that would be the one that contains $\mathbf{b}$. We explain these in more details in the following example.

\begin{ex}\label{ex4}
	Let $D=[A\mid \mathbf{b}]$, where $A$ is given in Example \ref{ex2} and we set 
	$\mathbf{b}=	15\F_3+9\F_9-3\F_{12}$. The last row (column) of the signature matrix $S_D=I-D^\dagger D$ of $D$  is as follows:
	\begin{equation*}
			\resizebox{\textwidth}{!}{
				$
	\left(\begin{array}{ccccccccccccccccccccccccccccccccccccccccc} 
	0.006 & 0.043 & -0.061 & 0 & -0.006 & 0.018 & -0.002 & -0.011 & 
	-0.002 & 0 & 0 & 0.02 & 0 & \cdots &0 & 0.006 \end{array}\right)
        $
     }
	\end{equation*}
Note that non-zero entries in this row represents columns of $A$ that correlate with $\b$.
	
\end{ex}

	\begin{algorithm}
	\label{alg-irr}
	\caption{Irrelevant feature removal}
	\begin{algorithmic}
		\STATE \textbf{Input:} a dataset $D=[A\mid \mathbf{b}]$ of size $m\times (n+1)$
		\STATE \textbf{Output:} Subsets of linearly dependent columns of $A$
		\STATE{  $S=I-D^{\dagger}D$}
		\STATE{  $\mathbf{s}=$ last row of $S$}
		\STATE{  $\text{TH}=\text{Ave(maxima}(\mid \mathbf{s}\mid))$ }
		\STATE{ $\text{Ind}=\{i\mid \mathbf{s}_i\geq \text{TH}, 1\leq i\leq n\}$}
		\STATE{Update $A$ as  $A=A_\text{Ind}$}
	\end{algorithmic}
\end{algorithm}

 Now, consider the linear system $A\textbf{x}=\textbf{b}$.  Since  $A\textbf{x}=\textbf{b}$ may not have exact solutions, instead we  find the unique solution with the smallest 2-norm that satisfy the least squares problem
\begin{align}
	\displaystyle  ||A\textbf{x} - \textbf{b}||_2,
\end{align}
over all $\textbf{x}$.  It is well-known that $\mathbf{x}=A^{\dagger}\textbf{b}$ is  the least squares with the smallest 2-norm, see \cite{golub2012matrix}.

\begin{ex}\label{ex3}
Consider the matrix  where $A$ given in Example \ref{ex2} and  set 
$\mathbf{b}=	15\F_3+9\F_9-3\F_{12}$. The least squares solutions to $A\mathbf{x}=\mathbf{b}$ is as follows:

	\begin{equation*}
\resizebox{\textwidth}{!}{
	$
\mathbf{x}=A^{\dagger}\mathbf{b}=	\left(\begin{array}{cccccccccccccccccccccccccccccccccccccccc} -0.9375 & -6.5625 & 9.375 & 0 & 0.9375 & -2.8125 & 0.33333 & 1.6667 & 0.33333 & 0 & 0 & -3.0 & 0 & \cdots & 0 \end{array}\right)
	$
}
\end{equation*}

Now we replace $\F_9$ with $\F'_9=\F_9+\mathbf{E}$ in $A$ where  $\mathbf{E}$ is a randomly generated column with normal distribution. Let us denote the perturbed $A$ with $\tilde{A}$.
and consider the least squares solutions to $\tilde A\mathbf{\tilde x}=\mathbf{b}$. We set 
$\Delta \mathbf{x}=\mid \mathbf{x}-\mathbf{\tilde x} \mid$ and note that 

	\begin{equation*}
\resizebox{\textwidth}{!}{
	$
	\Delta \mathbf{x}=\left(\begin{array}{cccccccccccccccccccccccccccccccccccccccc} 0 & 0 & 0 & 0 & 0 & 0 & 0.012821 & 0.064103 & 0.33333 & 0 & \cdots & 0 \end{array}\right)
	$
}
\end{equation*}
We note that features that are not in the same cluster as $\F_9$ will see no difference in their corresponding solutions of the original system and perturbed system.

We observe two phenomenon happening here. First, features that are irrelevant have their corresponding component in $\mathbf{x}$ equal to zero. That is,  features $\F_{11}, \F_{13}, \ldots, \F_{40}$ have their corresponding component to be zero in $\mathbf{x}$.  Second, if we perturb a feature $\F_i$ and consider the difference $\Delta \mathbf{x}$,  then features that are in different clusters than  $\F_i$ will have their corresponding component in $\Delta \mathbf{x}$ equal to zero. 
\end{ex}

 We also observe that this phenomenon occurs only for rank-deficient matrices as the following example shows. Let $A$ be an $m\times n$ matrix and  let us denote by $A_j$ the matrix obtained from $A$ by adding a random column vector $\mathbf{c}\in \mathbb{R}^m$ to $\mathbf{F_j}$. We realize that this kind of perturbation of $A$ can be expressed in terms of a rank-1 update of $A$. Consider the column vector  $\mathbf{e}_j\in \mathbb{R}^n$ as the $j$-th standard basis vector. It is easy to verify that $A_j=A+\mathbf{c}\mathbf{e}_j^T$.

\begin{ex}\label{ex-full-row}
	Let $A$ be a $8\times 11$ matrix constructed as follows. First we construct an $8\times 8$ random matrix and then set $\F_{9}=9\F_1, \F_{10}=10\F_2, \F_{11}=11\F_3$. We also set $\b=5\F_1+4\F_2-2\F_4$. Note that $\text{rank}(A)=8$. Here is an example of such an $A$:
	
		$$
\left(\begin{array}{ccccccccccc} 0.96 & 0.67 & 0.41 & 0.52 & 0.83 & 0.9 & 0.79 & 0.32 & 8.6 & 6.7 & 4.5\\ 0.3 & 0.13 & 0.64 & 0.62 & 1.0 & 0.37 & 0.93 & 0.65 & 2.7 & 1.3 & 7.1\\ 0.72 & 0.89 & 0.37 & 4.8e-3 & 0.4 & 0.012 & 0.86 & 0.58 & 6.5 & 8.9 & 4.0\\ 0.75 & 0.17 & 1.0 & 0.78 & 0.9 & 0.5 & 0.16 & 0.67 & 6.8 & 1.7 & 11.0\\ 0.19 & 0.54 & 0.71 & 0.21 & 0.53 & 0.44 & 0.081 & 0.79 & 1.7 & 5.4 & 7.8\\ 0.57 & 0.035 & 0.2 & 0.2 & 0.082 & 0.85 & 0.63 & 0.38 & 5.1 & 0.35 & 2.2\\ 0.061 & 0.81 & 0.83 & 0.79 & 0.42 & 0.5 & 0.097 & 0.079 & 0.55 & 8.1 & 9.1\\ 0.68 & 0.6 & 0.3 & 0.27 & 0.019 & 0.11 & 0.033 & 0.43 & 6.1 & 6.0 & 3.3 \end{array}\right)
	$$
	
	We perturb $\F_1$ and consider the perturbed system $(A+\mathbf{c}\mathbf{e}_1^T)\mathbf{\tilde x}=\b$, where $\c$ is a random column vector. We have 
	\begin{align*}
	\x&=\left(\begin{array}{ccccccccccc} 0.061 & 0.04 & 0 & -2.0 & 0& 0 & 0 & 0 & 0.55 & 0.4 & 0\end{array}\right)\\
	\mathbf{\tilde x}&= \left(\begin{array}{ccccccccccc} 0.54 & 0.035 & -6.1e-3 & -1.5 & 0.18 & 0.11 & 0.4 & 0.52 & 0.39 & 0.35 & -0.07 \end{array}\right) \\
	\Delta \x&=|\x -\mathbf{\tilde x}|{=}
\left(\begin{array}{ccccccccccc} 0.48 & 4.1e-3 & 6.1e-3 & 0.45 & 0.18 & 0.11 & 0.4 & 0.52 & 0.16 & 0.05 & 0.07 \end{array}\right).
		\end{align*}
	
		As we can see, even though $\F_1$ is only correlating with $\F_9$ but perturbations on $\F_1$ affects other components in $\Delta \x$. We can justify as follows: since $\text{rank}(A)=8$, any $\c\in \mathbb{R}^8$ is in the column space of $A$. Hence,  $\c$ can be written as a linear combination of $\F_1,\ldots, \F_{11}$. So $\F_1+\c$ will have a trace of other columns of $A$ as well. In turn, we can expect to see that some component of  the least square solution of  $A\x=\b$ can change.

\end{ex}

 Meyer in \cite{meyer1973generalized} provides the  pseudo-inverse of $A+\mathbf{c}\mathbf{e}_j^T$ that is expressed  in terms of $A^\dagger$ and some other matrices. We include this result here for convenience. 

Let $A\in \mathbb{C}^{m\times n}$, $\textbf{c}\in \mathbb{C}^m$ and $\textbf{d}\in \mathbb{C}^n$. Consider the following definitions:  $\textbf{k}=A^\dagger \textbf{c}, \textbf{h}=\textbf{d}^TA^\dagger, \textbf{u}=(I-AA^\dagger)\textbf{c}, \textbf{v}=\textbf{d}^T(I-A^\dagger A)$, and $\beta= 1+\textbf{d}^TA^\dagger \textbf{c}$. Throughout, we denote Moore- Penrose inverse of a vector $\x$ by 
$$
\x^\dagger=\frac{\x^T}{|| \x ||^2}.
$$

\begin{theorem}[Pseudo-Inverse of rank-1 update ]\label{psi} 
	Let $A\in \mathbb{C}^{m\times n}$, $\c\in \mathbb{C}^m$ and $\d\in \mathbb{C}^n$. Then the generalized inverse of $(A+\textbf{c}\textbf{d}^T)$ is as follows:
	\begin{enumerate}
		\item If $\u \neq 0$ and $\v \neq 0$, then $(A+\textbf{c}\textbf{d}^T)^\dagger = 
		A^\dagger - \k\u^{\dagger} - \v^{\dagger} \h + \beta \v^{\dagger} \u^{\dagger}$.
		\item If $\u = 0$, $\v \neq 0$, and $\beta = 0$, then 
		$(A+\textbf{c}\textbf{d}^T)^\dagger = A^\dagger - \k\k^\dagger A^\dagger - \v^\dagger \h$.
		\item If $\textbf{u} = 0$ and  $\beta \neq 0$, then 
		$(A+\textbf{c}\textbf{d}^T)^\dagger = A^\dagger + \dfrac{1}{\beta}\textbf{v}^T\textbf{k}^TA^\dagger - \dfrac{\beta}{\sigma_1}\textbf{p}_1\textbf{q}_1^T$, where $\textbf{p}_1= - (\dfrac{\parallel \textbf{k} \parallel^ 2}{\beta}\textbf{v}^T + \textbf{k} ), \textbf{q}_1^T= - (\dfrac{\parallel \textbf{v} \parallel^ 2}{\beta}\textbf{k}^TA^\dagger + \textbf{h} )$, and $\sigma_1= \parallel \textbf{k} \parallel ^2 \parallel \textbf{v} \parallel ^2 + \mid \beta\mid ^2$.
		
		\item If $\u\neq 0$, $\v = 0$, and $\beta = 0$, then 
		$(A+\textbf{c}\textbf{d}^T)^\dagger = A^\dagger - A^\dagger \h^\dagger \h - \k\u^\dagger$.
		\item If $\v = 0$ and  $\beta \neq 0$, then $(A+\textbf{c}\textbf{d}^T)^\dagger = A^\dagger + \dfrac{1}{\beta}A^\dagger \h^T\u^T - \dfrac{\beta}{\beta_2}\p_2\q_2^T$, where
		 $\p_2= - (\dfrac{\parallel \u \parallel^ 2}{\beta}A^\dagger \h^T + \k ), \q_2^T= - (\dfrac{\parallel \h \parallel^ 2}{\beta}\u^T + \h )$, 
		 and $\beta_2= \parallel \h \parallel ^2 \parallel \u \parallel ^2 + \mid \beta\mid ^2$.
		\item If $\u=0 , \v=0$ and $\beta = 0$, then 
		$(A + \c\d^T)^\dagger = A^\dagger - \k\k^\dagger A^{\dagger} - A^\dagger \h^\dagger\h + (\k^\dagger A^\dagger \h^\dagger)\k\h$.
	\end{enumerate}
\end{theorem}

Note that $\mathbf{c}\in \text{R}(A)$ if and only if $\mathbf{u}=0$ and 
$\mathbf{d}\in \text{R}(A^T)$ if and only if $\mathbf{v}=0$. Note that, by Theorem  \ref{indp-thm2},
$\mathbf{F}_j$ is independent of the rest of columns of $A$ if and only if 
$\mathbf{v}=\textbf{e}_j^T(I-A^\dagger A)=0$. 

\begin{theorem}\label{indp-soln}
	Suppose that column $\mathbf{F}_j$ of $A$ is independent of the rest of columns of $A$. Let $\c\in \mathbb{C}^m$ such that $\beta= 1+\e_j^TA^\dagger \textbf{c}\neq 0$. Let
$\mathbf{x}=A^\dagger\mathbf{b}$ and 
$\tilde{\mathbf{x}}=(A+\mathbf{c}\mathbf{e}_j^T)^\dagger\mathbf{b}$. Then 
$|| \mathbf{x} - \tilde{\mathbf{x}}||\leq\dfrac{|x_j| || A^\dagger||  }{||\h||^2}$.
\end{theorem}
\begin{proof}
By Theorem \ref{indp-thm2}, we have $\mathbf{v}=0$. Note that    $\mathbf{u}^T\b=(I-(A^\dagger)^TA^T)A\x$. Using the SVD, we can see that $(I-(A^\dagger)^TA^T)A=0$. So, $\mathbf{u}^T\b=0$. Furthermore, 
\begin{align}\label{khb}
\mathbf{hb}=&\mathbf{e}_j^TA^\dagger \mathbf{b}=\mathbf{e}_j^T\mathbf{x}=x_j
\end{align}
Note that Part (5) of Theorem \ref{indp-thm2} applies. Hence, 
\begin{align*}
|| \mathbf{x} - \tilde{\mathbf{x}}||&=|| A^\dagger\mathbf{b} - (A+\mathbf{c}\mathbf{e}_j^T)^\dagger \mathbf{b} ||\\
&=||- \dfrac{\beta}{\beta_2} (\dfrac{\parallel \u \parallel^ 2}{\beta}
A^\dagger \h^T + \k )\h \b ||=\dfrac{\parallel \u \parallel^ 2}{\beta_2}||  A^\dagger \h^T \h\b +\k\h\b ||\\
&=\dfrac{|x_j| \parallel \u \parallel^ 2}{\beta_2}||  A^\dagger \h^T  +A^\dagger \c ||
\leq \dfrac{|x_j| \parallel \u \parallel^ 2}{||\u ||^2||\h||^2}||  A^\dagger \e_j^T A^\dagger  +A^\dagger \c ||
\leq \dfrac{|x_j| || A^\dagger||  }{||\h||^2}.
\end{align*}

Note that the condition $\beta \neq 0$ in  Theorem \ref{indp-soln} can be easily satisfied because we can choose $\c$ randomly and if $1+\e_j^TA^\dagger \textbf{c}= 0$ then we can simply rescale $\c$. Theorem \ref{indp-soln}  implies that if $|x_j|$ is very small then the least squares solutions of $A\x=\b$ are insensitive to   perturbations in $\F_j$. This way we can distinguish these columns as irrelevant.

\end{proof}

\begin{theorem}\label{pert}
	Let $A\in \mathbb{R}^{m\times n}$ be a matrix of rank $\rho<\min(m,n)$. Let $\mathbf{c}\in \mathbb{R}^n$ so that $\mathbf{c}\notin \text{R}(A)$. Let $\mathbf{x}$ and $\mathbf{\tilde x}$ be the least squares solutions to 
	$A\mathbf{x}=\mathbf{b}$ and $(A+\mathbf{c}\mathbf{e}_j^T)\mathbf{\tilde x}=\mathbf{b}$, respectively.
	
\begin{enumerate}
	\item 	 If $\F_i$ and  $\F_j$ are in different clusters then  $x_i=\tilde x_i$.	
	\item $\parallel \x-\mathbf{\tilde{x}} \parallel \leq 2|x_j| $. In particular, if $|x_j| $ is very small then $\x=\mathbf{\tilde{x}}$.
\end{enumerate}
\end{theorem}
\begin{proof}
	Without loss of generality we assume that $j=1$ and that $\F_1, \ldots, \F_t$ ($t\geq 1$) is the set of all columns in the cluster of $\F_1$.
	Then 
	$$
	\mathbf{v}=\mathbf{e}_j^T(I-A^\dagger A)=		\left[\begin{array}{cccccccccccccccccccccccccccccccccccccccc} P_{1,1}  & \cdots & P_{1,t} & 0 & & \cdots & 0  \end{array}\right].
	$$
 So,  by Part (1) of Theorem \ref{indp-thm2},  we have
	\begin{align*}
	\mathbf{\tilde x}=(A+\mathbf{c}\mathbf{e}_j^T)^{\dagger} \mathbf{b}=& 
	A^\dagger \b  - \mathbf{k}\mathbf{u}^\dagger\b - \mathbf{v}^\dagger \mathbf{h}\b + \beta \mathbf{v}^\dagger \mathbf{u}^\dagger\b
	\end{align*}
Similarly as in the proof of Theorem \ref{indp-soln},  we  have $\mathbf{u}^T\b=0$ and $ \mathbf{hb}=x_j$.
	Note that $i>t$ and so $	\mathbf{e}_i^T \mathbf{v}^T=0$. Hence 
	
		\begin{align*}
	\tilde x_i -x_i	&= \mathbf{e}_i^T(- \mathbf{k}\mathbf{u}^T\b - \mathbf{v}^T \mathbf{h}\b + \beta \mathbf{v}^T \mathbf{u}^T\b)	=0.
	\end{align*}
Furthermore, $\parallel \Delta \x \parallel=\parallel \x-\mathbf{\tilde x}\parallel= \parallel \mathbf{v}^T \mathbf{h}\b \parallel=|x_j| \parallel \v^T\parallel$. Note that 
$\parallel \v^T\parallel=\parallel \textbf{e}_j^T(I-A^\dagger A) \parallel \leq \parallel I-A^\dagger A\parallel \leq \parallel I\parallel+\parallel A^\dagger A\parallel \leq2$.
Hence, $\parallel \Delta \x \parallel\leq 2|x_j| $.

\end{proof}

Note that the condition $\mathbf{c}\notin \text{R}(A)$ in  Theorem \ref{pert} is satisfied with probability 1 when $A$ is not full-column rank. Indeed, suppose $W$ is a subspace of dimension $n-1$ in $\mathbb{R}^n$.
Without loss of generality, we can assume $\e_1,\ldots, \e_{n-1}\in W$. Suppose we pick a random vector $\c\in \mathbb{R}^n$ and we want to see the probability that $\c\in W$. Clearly $\c\in W$ if and only if 
the $n$-th component of $\c$ is zero. However, since we are choosing $\c$ randomly, that probability is  zero.

We now sketch a feature selection algorithm based on Theorem \ref{pert}. First note that Part (2) of the theorem tells us if a feature is irrelevant ($|x_j|$ is very small ) then solutions of
	$A\mathbf{x}=\mathbf{b}$ and $(A+\mathbf{c}\mathbf{e}_j^T)\mathbf{\tilde x}=\mathbf{b}$ are the same. So, perturbing irrelevant columns does not alter the least square solutions. In fact we can detect all irrelevant features all at once. To do so, we choose a $m\times n$ matrix 
	$E$ such that each column of $E$ is not in the column space of $A$. Then we look at the solutions of $A\mathbf{x}=\mathbf{b}$ and  $(A+E)\mathbf{\tilde x}=\mathbf{b}$. Then a column $\F_i$ is irrelevant if and only if 
	$| x_i-\tilde{x}_i|$ is zero. Once we detect irrelevant columns, we update $A$ by  removing irrelevant columns.  Then we can detect correlations between columns using Part 1 of Theorem \ref{pert}. In light of feature selection algorithm that we offered in  Algorithms 3.1 and 3.2,  one might wonder why we are offering another algorithm based on Theorem \ref{pert}. In real datasets, 
	we do not have exact relations and one has to approximate by setting a threshold. For example, how do we set a threshold to remove irrelevant columns? Well, if the threshold is hard, we may omitting some important features that are just on the borderline.   Our idea is to  combine the two algorithms together and  set soft thresholds at each step. We shall pursue the details in a separate paper.

\section*{Acknowledgments}The research of the second author  was supported by NSERC of Canada under grant \# RGPIN 418201. I would like to thank Scott Maclachlan for reading an early draft of this paper and making valuable comments and suggestions.

\bibliographystyle{elsarticle-num}

\end{document}